\documentclass[10pt]{article}

\usepackage[utf8]{inputenc}
\usepackage[T1]{fontenc}
\usepackage{charter}                         
\usepackage{amsmath,amssymb,amsthm}
\usepackage[a4paper,top=2.4cm,bottom=2.6cm,left=2.4cm,right=2.4cm]{geometry}
\usepackage{booktabs}
\usepackage{graphicx}
\usepackage{float}
\usepackage{xcolor}
\usepackage{caption}
\usepackage{titlesec}
\usepackage{fancyhdr}
\usepackage{enumitem}
\usepackage{authblk}
\usepackage{hyperref}
\definecolor{linkblue}{rgb}{0.10,0.30,0.55}
\hypersetup{colorlinks=true,linkcolor=linkblue,citecolor=linkblue,urlcolor=linkblue}

\DeclareCaptionLabelSeparator{bar}{\ \textbar\ }
\titleformat{\section}{\large\bfseries}{\thesection.}{0.5em}{}
\titleformat{\subsection}{\normalsize\bfseries}{\thesubsection.}{0.5em}{}
\titlespacing*{\section}{0pt}{1.5ex plus .3ex minus .2ex}{0.7ex}
\titlespacing*{\subsection}{0pt}{1.1ex plus .2ex}{0.4ex}

\newcommand{\runningtitle}{An Integrable Token-Mixing Layer from the Generalized Yang--Baxter Equation}
\fancypagestyle{titlestyle}{%
  \fancyhf{}%
  \fancyhead[L]{\small\bfseries Preprint}%
  \fancyhead[R]{\small\today}%
  \fancyfoot[C]{\small\thepage}%
}

\renewenvironment{abstract}{\vspace{0.2em}\par\noindent\bfseries\ignorespaces}{\par\vspace{0.6em}\normalfont}

\newtheorem{lemma}{Lemma}

\newtheorem{remark}{Remark}

\newcommand{\R}{\mathbb{R}}

\newcommand{\one}{\mathbb{1}}
\newcommand{\tr}{\operatorname{tr}}
\newcommand{\acomm}[2]{\{#1,#2\}}
\newcommand{\comm}[2]{[#1,#2]}

\begin{document}
\thispagestyle{titlestyle}

\begin{flushleft}
{\LARGE\bfseries An Integrable Token-Mixing Layer\\[2pt]
from the Generalized Yang--Baxter Equation\par}
\vspace{0.9em}
{\large Snigdha Chandan Khilar\textsuperscript{1}}\par
\vspace{0.25em}
{\small\textsuperscript{1}\,Independent Researcher \hspace{0.8em}\textbar\hspace{0.8em} \texttt{snkhilar@gmail.com}}\par
\end{flushleft}
\let\thefootnote\relax\footnotetext{Correspondence: \texttt{snkhilar@gmail.com}}
\vspace{0.3em}

\begin{abstract}
We introduce \textbf{YB-Mixer}, a sequence token-mixing layer derived from the
free-fermion / generalized Yang--Baxter structure recently used to construct
\emph{hidden} transverse-field Ising models. The design rests on a single transferable
principle from integrable systems: a \emph{local} algebraic constraint on adjacent
operations can certify \emph{global} computational guarantees, independently of the
representation. Concretely, the \emph{Ising exchange algebra} (an extraspecial
$2$-group relation) certifies (i) a free-fermionic structure that makes the mixer an
exactly norm-preserving orthogonal map, and (ii) commuting transfer matrices that make
inference \emph{order-free} and \emph{variable-budget} (``anytime'').
We provide a complete, reproducible empirical pipeline of seven experiments.
We verify the generalized Yang--Baxter equation (gYBE) numerically to machine precision;
prove and verify that the YBE constraint reduces to a well-conditioned algebraic surrogate,
making integrable gates efficiently learnable; build a brick-wall YB-Mixer layer that is
exactly norm-preserving and depth-stable (Jacobian condition number $=1$ at all depths);
verify commuting transfer matrices and the resulting schedule-invariant inference;
train an integrable-flow model end-to-end that matches a self-attention baseline on a
long-range transport task at $\sim\!3.3\times$ fewer parameters; and demonstrate exact
order-free, variable-budget inference that attention lacks. We compare against orthogonal RNN, diagonal state-space, attention, and nonlinear-mixer
baselines, finding YB-Mixer matches or exceeds the structured baselines on long-range memory
at fewer parameters while honestly lagging a nonlinear mixer on content-dependent recall.
Finally we show that the length-generalization failure of a \emph{local} generator is fixed
by a \emph{spectral} (non-local, circulant) generator that remains orthogonal and commuting:
trained at $L{=}16$ it generalizes to $L{=}64$ with roughly flat accuracy. Scaled to
$\sim\!2.5$M parameters across five downstream tasks against \emph{properly-tuned} baselines
(an S4D-Lin/HiPPO SSM, LRU, Transformer, and FNet), the orthogonal spectral mixer is competitive
with the strongest baseline---best or tied on three of five tasks at the fewest parameters,
reaching $84.8\%$ on sequential-CIFAR (LRA-Image) versus $51$--$72\%$ for the identically-scaffolded
Transformer, LRU, and FNet---and is one of only two mixers that solve long-range token retrieval
(Induction Heads) exactly. Code: \url{https://github.com/nssprogrammer/yb-mixer}.
\end{abstract}

\section{Introduction}\label{sec:intro}

Modern sequence models are built from \emph{token mixers} that exchange information across
positions: self-attention~\cite{vaswani}, MLP-style mixers~\cite{mlpmixer}, and structured
state-space models (SSMs)~\cite{s4,mamba,lru}. Two recurrent practical concerns are
\emph{stability} (gradients should neither vanish nor explode with depth) and
\emph{inference flexibility} (the ability to spend variable compute at test time).
Orthogonal and unitary recurrent layers~\cite{unitaryrnn} address stability by construction.
Here we ask whether the much richer toolkit of \emph{quantum integrability}---which is, at
heart, a theory of when many operations \emph{commute}---can be exported to design token
mixers with provable structure.

Our starting point is a recent construction of \emph{hidden} transverse-field Ising models
(TFIMs) from the generalized Yang--Baxter equation~\cite{sinha2026}. That work shows that
seemingly interacting multi-site spin chains are secretly free-fermionic, integrable, and
governed by the \emph{Ising exchange algebra} of their Hamiltonian densities. The key
lesson, abstracted away from physics, is a design pattern:

\begin{quote}
\emph{A purely local algebraic relation between neighbouring operations can certify a global,
representation-independent computational property---here, exact diagonalizability
(orthogonality) and commuting families of operators (order-freedom).}
\end{quote}

We turn this pattern into a concrete neural layer, \textbf{YB-Mixer}, and validate every link
of the chain numerically. Our contributions are:
\begin{enumerate}[leftmargin=*,itemsep=1pt]
\item \textbf{A verified integrable primitive} (\S\ref{sec:exp-primitive}). We construct the
gYBE $R$-matrix from extraspecial-$2$-group generators and verify the
$(d,6,3)$-gYBE to machine precision.
\item \textbf{A learnability reduction} (\S\ref{sec:exp-learn}). We prove that for the
Baxterized ansatz $R(\lambda)=\one+\tan(\lambda)M$ the braided YBE residual vanishes
\emph{iff} $M^2=\one$ and neighbouring embeddings anticommute, and show that direct
residual minimization is ill-conditioned whereas the equivalent algebraic surrogate
reliably yields integrable gates.
\item \textbf{A norm-preserving, depth-stable mixer} (\S\ref{sec:exp-stab}). The free-fermion
gate acts as an orthogonal map on token features; a brick-wall of such gates has Jacobian
condition number exactly $1$ at all depths.
\item \textbf{Commuting transfer matrices and (scoped) anytime inference} (\S\ref{sec:exp-tau},
\S\ref{sec:exp-anytime}). We verify $\comm{\tau(\lambda)}{\tau(\mu)}\approx 0$ and show that an
integrable-\emph{flow} model supports \emph{exact} order-free, variable-budget inference. We are
explicit that this is a property of the one-parameter group, holds only for the integrable flow
(not arbitrary nets containing a YB-Mixer layer), and is one realization---with an exactness and
order-freedom guarantee---of the broader adaptive-computation idea~\cite{branchynet,calm,universaltransformer,matryoshka}.
\item \textbf{A rigorous, honest empirical study} (\S\ref{sec:exp-train},
\S\ref{sec:exp-baselines}). YB-Mixer matches a self-attention baseline on a long-range
transport task at far fewer parameters (multi-seed), with a principled initialization recipe;
we document a length-generalization limitation rooted in free-fermion dispersion.

\item \textbf{Baselines and a length-generalization fix} (\S\ref{sec:exp-baselines2},
\S\ref{sec:exp-baselines}). Against orthogonal RNN, diagonal SSM, attention, and a nonlinear
mixer, YB-Mixer matches or exceeds the structured baselines on long-range memory at fewer
parameters, and trails only the nonlinear mixer on content recall. We further resolve the
dispersion-driven length-generalization failure with a \emph{spectral} generator that stays
orthogonal and commuting and generalizes to $4\times$ the training length.
\item \textbf{Scaled benchmarks} (\S\ref{sec:exp-scaled}). At $\sim$2.5M parameters against
properly-tuned baselines (S4D-Lin/HiPPO, LRU, Transformer, FNet), the orthogonal spectral mixer
is best or tied-best on three of five tasks at the fewest parameters---tying the tuned SSM on
sequential-CIFAR (LRA-Image, $84.8\%$ vs $51$--$72\%$ for Transformer/LRU/FNet) and solving
Induction-Heads retrieval exactly---while honestly trailing the tuned SSM on IMDB and ListOps.
\end{enumerate}

All experiments are small, fully reproducible, and provided as standalone scripts in the released code. We are candid that this is a controlled-task study: it establishes the
architecture and verifies its properties, and does not claim benchmark-scale accuracy
(\S\ref{sec:limitations}).

\section{Background theory}\label{sec:background}

We collect the integrable-systems machinery we use. Readers familiar with the
transverse-field Ising model, Jordan--Wigner fermionization, and the quantum inverse
scattering method may skim to \S\ref{sec:background-gybe}.

\subsection{The transverse-field Ising model and the Ising exchange algebra}

The one-dimensional spin-$\tfrac12$ TFIM on $N$ sites is
\begin{equation}
H_{\mathrm{TFIM}} \;=\; -\,g\sum_{j} Z_j \;-\; \sum_{j} X_j X_{j+1},
\label{eq:tfim}
\end{equation}
where $X_j,Y_j,Z_j$ are Pauli operators acting on site $j$. Define the local
\emph{Hamiltonian densities}
\begin{equation}
h^{z}_j \;=\; Z_j, \qquad h^{xx}_j \;=\; X_j X_{j+1}.
\end{equation}
A direct computation shows they obey the \emph{Ising exchange algebra}:
\begin{equation}
\begin{aligned}
&\comm{h^{z}_j}{h^{z}_k} = \comm{h^{xx}_j}{h^{xx}_k} = 0, \qquad
\comm{h^{z}_j}{h^{xx}_k}=0 \ \ (j \ne k,\,k{+}1),\\[2pt]
&\acomm{h^{z}_j}{h^{xx}_j} = \acomm{h^{z}_{j+1}}{h^{xx}_k}=0, \qquad
(h^{z}_j)^2 = (h^{xx}_j)^2 = \one .
\end{aligned}
\label{eq:ising-exchange}
\end{equation}
The crucial fact~\cite{minami2016,sinha2026} is that \eqref{eq:ising-exchange}
\emph{alone}---independent of the matrix realization---forces a free-fermionic spectrum.
A \emph{local} relation between neighbours thus certifies a \emph{global} structural property.
This representation-independence is what we will exploit.

\subsection{Jordan--Wigner fermionization}

The Jordan--Wigner (JW) transformation maps spins to Majorana fermions,
\begin{equation}
\gamma_{2j-1} = \Big(\prod_{k<j} Z_k\Big) X_j, \qquad
\gamma_{2j} = \Big(\prod_{k<j} Z_k\Big) Y_j, \qquad
\acomm{\gamma_i}{\gamma_j} = 2\,\delta_{ij}\,\one .
\label{eq:jw}
\end{equation}
Under \eqref{eq:jw} the TFIM \eqref{eq:tfim} becomes a quadratic (free) Majorana
Hamiltonian $H = i\,g\sum_j \gamma_{2j-1}\gamma_{2j} + i\sum_j \gamma_{2j}\gamma_{2j+1}$,
diagonalizable by a Bogoliubov (orthogonal) rotation. An operator is \emph{free-fermionic}
precisely when it is \emph{quadratic} in the $\gamma$'s; its action is then completely
determined by an antisymmetric \emph{single-particle} matrix, an $O(\dim)$ object rather than
the exponential many-body operator. This single-particle reduction is the bridge to a
classical neural layer (\S\ref{sec:layer}).

\subsection{The generalized Yang--Baxter equation}\label{sec:background-gybe}

The ordinary Yang--Baxter equation~\cite{yang1967,baxter1982} is the consistency condition for
factorized scattering of two-body interactions. Its \emph{generalized}
$(d,\ell,m)$ form~\cite{rowell2010} allows $R$-matrices supported on $\ell$ adjacent sites,
shifted by $m$:
\begin{equation}
R_{1\cdots \ell}(\lambda)\,R_{(1+m)\cdots(\ell+m)}(\lambda{+}\mu)\,R_{1\cdots \ell}(\mu)
=
R_{(1+m)\cdots(\ell+m)}(\mu)\,R_{1\cdots \ell}(\lambda{+}\mu)\,R_{(1+m)\cdots(\ell+m)}(\lambda),
\label{eq:gybe}
\end{equation}
an operator equation on $\bigotimes_{j=1}^{\ell+m} \mathcal{H}_d$, written here in
\emph{braided} (additive) form. The construction of~\cite{sinha2026} uses multi-site operators
$M_j$ built from generators of \emph{extraspecial $2$-groups}, satisfying
\begin{equation}
M_j^2 = \one, \qquad \acomm{M_j}{M_{j+1}} = 0, \qquad
\comm{M_j}{M_k} = 0 \ \ (|j-k|\ge 2),
\label{eq:m-algebra}
\end{equation}
together with the \emph{Baxterized} $R$-matrix
\begin{equation}
R(\lambda) \;=\; \one + \tan(\lambda)\, M .
\label{eq:baxter}
\end{equation}
The spectral parameter enters through the tangent because $M^2=\one$: substituting
\eqref{eq:baxter} into \eqref{eq:gybe} and demanding nontrivial solutions yields the
functional equation
\begin{equation}
a(\lambda_1{+}\lambda_3) = \frac{a(\lambda_1)+a(\lambda_3)}{1-\kappa\, a(\lambda_1)a(\lambda_3)},
\qquad M^2 = \kappa\,\one,
\label{eq:funceq}
\end{equation}
solved by $a(\lambda)=\tan(\lambda)/\sqrt{\kappa}$ (here $\kappa=1$, the tangent addition law).
We give the elementary reduction underlying this in Lemma~\ref{lem:reduction} below, which is
also what makes integrable gates learnable.

\subsection{Quantum inverse scattering and commuting transfer matrices}\label{sec:background-qism}

Given an $R$-matrix solving the (non-braided) YBE, the quantum inverse scattering method
(QISM)~\cite{korepin1993} builds a one-parameter family of mutually commuting operators.
With an auxiliary space $a$, the \emph{monodromy} and \emph{transfer} matrices are
\begin{equation}
T_a(\lambda) = R_{a,N}(\lambda)\cdots R_{a,1}(\lambda), \qquad
\tau(\lambda) = \tr_a\, T_a(\lambda),
\label{eq:transfer}
\end{equation}
and the YBE/$RTT$ relation implies
\begin{equation}
\comm{\tau(\lambda)}{\tau(\mu)} = 0 \qquad \forall\,\lambda,\mu .
\label{eq:commuting-tau}
\end{equation}
Equation~\eqref{eq:commuting-tau} is the algebraic heart of integrability: an entire family of
``forward passes'' indexed by the spectral parameter mutually commute. In \S\ref{sec:layer}
we read this as \emph{order-freedom} of inference.

\subsection{The boost operator (briefly)}

The conserved charges $I_{r+1}$ generated by $\tau(\lambda)$ can be obtained from a single
\emph{boost operator} $B=\sum_j j\,M_j$ via the recursion
$I_{r+1}=\tfrac1r\comm{B}{I_r}$~\cite{sogo1983,sinha2026}. Each charge is a range-$r$
bilinear with a string of conserved central elements between its endpoints. We do not use the
boost tower in our experiments but note it as a route to multi-scale, mutually compatible
features (\S\ref{sec:conclusion}).

\section{From integrable algebra to a neural layer}\label{sec:layer}

\begin{figure}[tb]
\centering
\includegraphics[width=\linewidth]{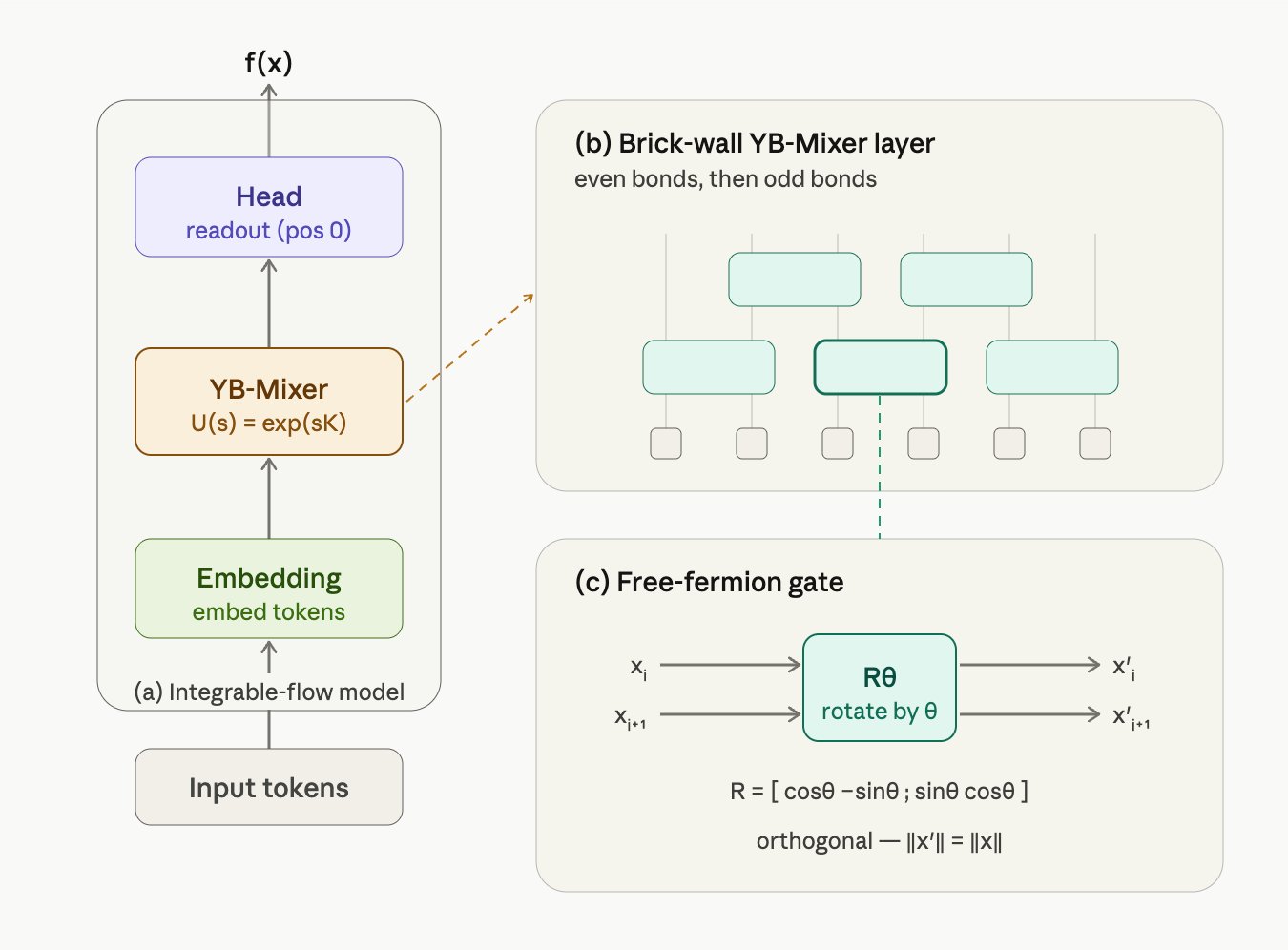}
\caption{\textbf{The YB-Mixer architecture.} (a)~The integrable-flow model (Eq.~\ref{eq:flow}): the
input sequence is embedded, mixed by a single orthogonal flow $U(s)=\exp(sK)$ generated by a
learned antisymmetric generator $K$, and read out by a small nonlinear head applied once at
position~$0$. (b)~The brick-wall YB-Mixer layer, the discrete realization of the flow: two-token
integrable gates act on the even bonds $(1,2),(3,4),\dots$ and then the odd bonds
$(2,3),(4,5),\dots$; stacking $\Theta(L)$ such layers produces a light cone that couples the
entire sequence. (c)~The free-fermion gate, the integrable primitive (Eq.~\ref{eq:rotmix}): a
per-channel $2{\times}2$ rotation by angle~$\theta$. Because the gate is quadratic in Majorana
operators, its single-particle action is an orthogonal matrix, so the mixer is exactly
norm-preserving and depth-stable, with Jacobian condition number~$1$ at all depths
(Table~\ref{tab:stability}).}
\label{fig:arch}
\end{figure}

The global structure of YB-Mixer is shown in Figure~\ref{fig:arch}. The design reads the
integrable operator algebra as a wiring diagram: local (anti)commutation of adjacent gates
certifies a global free-fermionic---hence orthogonal---mixing rule, and the Baxterized gate
supplies a continuous mixing-strength dial~$\lambda$.

\subsection{Design principle}

YB-Mixer is obtained by reading the operator algebra as a wiring diagram with guarantees (Table~\ref{tab:mapping}):

\begin{table}[H]
\centering
\begin{tabular}{ll}
\toprule
\textbf{Integrable object} & \textbf{Neural-layer reading}\\
\midrule
Local (anti)commutation \eqref{eq:ising-exchange},\eqref{eq:m-algebra} & constraint on how adjacent gates interact\\
``algebra $\Rightarrow$ free-fermionic'' & ``local rule $\Rightarrow$ orthogonal, norm-preserving mixing''\\
Baxterized $R(\lambda)$ \eqref{eq:baxter} & a mixing gate with a continuous strength dial $\lambda$\\
Commuting $\tau(\lambda)$ \eqref{eq:commuting-tau} & order-free / variable-budget (anytime) inference\\
\bottomrule
\end{tabular}
\caption{From integrable structure to neural-layer design: each algebraic property is read as a guarantee on the mixing layer.}
\label{tab:mapping}
\end{table}

\subsection{The free-fermion reduction makes the gate orthogonal}

A free-fermion gate is \emph{quadratic} in Majoranas and therefore acts on the
\emph{single-particle} space as an orthogonal matrix. For example, the two-qubit gate
$M = X\otimes Y$ equals, under \eqref{eq:jw}, the Majorana bilinear
\begin{equation}
X\otimes Y = -\,i\,\gamma_2\gamma_4 ,
\end{equation}
whose single-particle action is a rotation in the $(\gamma_2,\gamma_4)$ plane. Consequently a
brick-wall of such gates is an \emph{orthogonal token mixer}: norm-preserving by construction,
with Jacobian singular values identically $1$.

\subsection{The brick-wall YB-Mixer layer}

Let $X\in\R^{B\times L\times C}$ be a batch of $L$-token sequences with $C$ feature channels.
A YB-Mixer layer applies a two-token integrable gate on the even bonds
$(1,2),(3,4),\dots$ and then the odd bonds $(2,3),(4,5),\dots$. In the simplest free-fermion
instantiation the gate is a per-channel rotation by an angle $\theta$,
\begin{equation}
\begin{pmatrix} x'_i \\ x'_{i+1}\end{pmatrix}
=
\begin{pmatrix} \cos\theta & -\sin\theta \\ \sin\theta & \cos\theta \end{pmatrix}
\begin{pmatrix} x_i \\ x_{i+1}\end{pmatrix},
\label{eq:rotmix}
\end{equation}
which is exactly orthogonal and integrable. Stacking $\Theta(L)$ such layers yields a light
cone covering the whole sequence. Crucially, the following reduction makes
\emph{learning} an integrable gate well-posed.

\begin{lemma}[YBE reduction]\label{lem:reduction}
Let $M_A,M_B$ be Hermitian with $M_A^2=M_B^2=\one$ and $\acomm{M_A}{M_B}=0$, and set
$R_A(\lambda)=\one+\tan(\lambda)M_A$, $R_B(\lambda)=\one+\tan(\lambda)M_B$. Then the braided
YBE
\[
R_A(\lambda)R_B(\lambda{+}\mu)R_A(\mu) = R_B(\mu)R_A(\lambda{+}\mu)R_B(\lambda)
\]
holds for all $\lambda,\mu$ if and only if
$\tan(\lambda{+}\mu)=\dfrac{\tan\lambda+\tan\mu}{1-\tan\lambda\tan\mu}$ (the tangent
addition law).
\end{lemma}

\begin{proof}
Write $x=\tan\lambda$, $z=\tan\mu$, $y=\tan(\lambda{+}\mu)$. Using $M_A^2=M_B^2=\one$ and
$M_BM_A=-M_AM_B$, expand both sides:
\[
\text{LHS} = (1{+}xz) + (x{+}z)M_A + y(1{-}xz)M_B + y(x{-}z)\,M_AM_B,
\]
\[
\text{RHS} = (1{+}xz) + y(1{-}xz)M_A + (x{+}z)M_B + y(x{-}z)\,M_AM_B.
\]
Since $\one,M_A,M_B,M_AM_B$ are linearly independent, equality holds iff $x+z=y(1-xz)$, i.e.
$y=(x{+}z)/(1{-}xz)$.
\end{proof}

\begin{remark}[Learnability]
Lemma~\ref{lem:reduction} shows the integrability constraint on a learnable gate is
\emph{exactly} ``$M^2=\one$ and neighbouring embeddings anticommute''. Rather than minimize the
cubic YBE residual (ill-conditioned; see \S\ref{sec:exp-learn}), we parameterize $M=UDU^\dagger$
with $D=\mathrm{diag}(1,1,-1,-1)$ (so $M^2=\one$, $\tr M=0$ by construction) and minimize only
the anticommutator $\|M_{12}M_{23}+M_{23}M_{12}\|^2$. This well-conditioned surrogate reliably
recovers integrable gates.
\end{remark}

\subsection{The integrable-flow model and anytime inference}\label{sec:flow}

\begin{figure}[tb]
\centering
\includegraphics[width=0.92\linewidth]{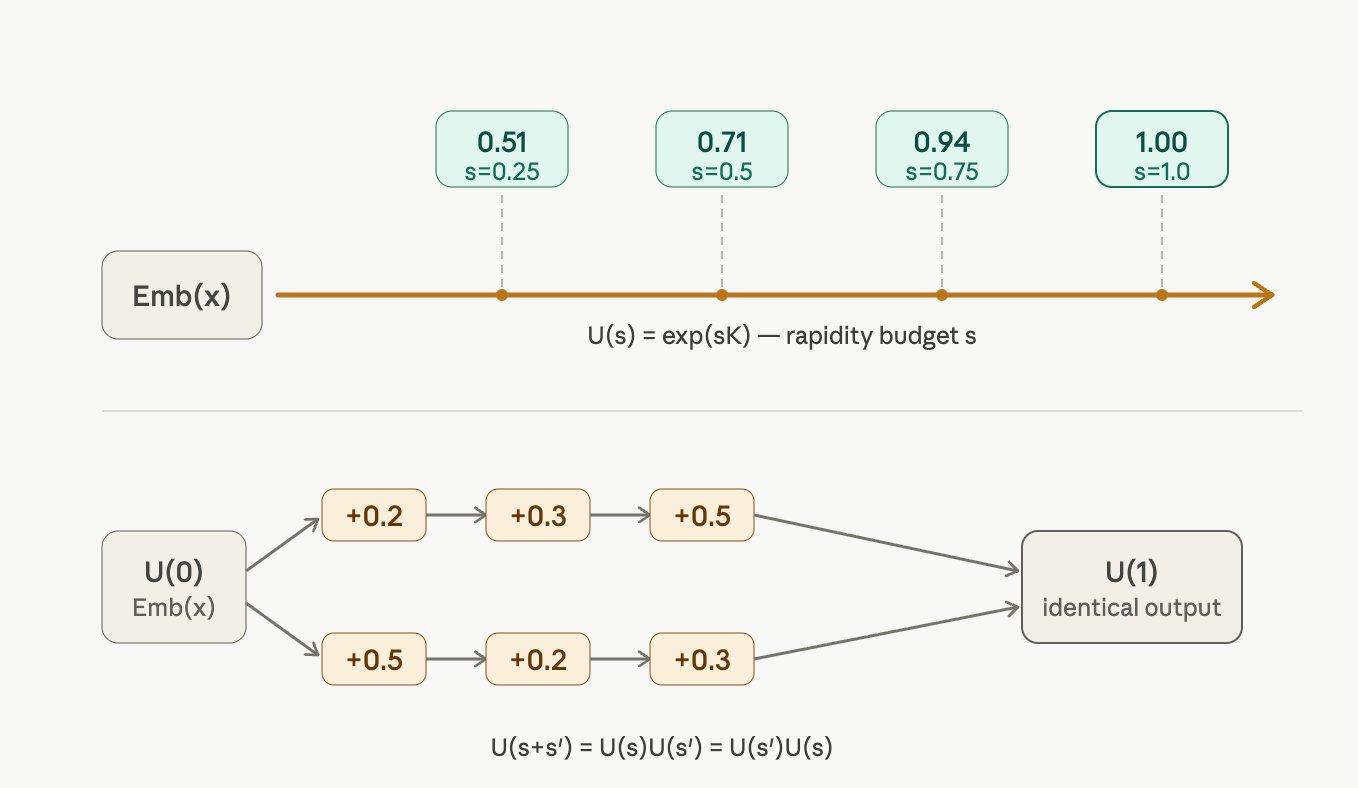}
\caption{\textbf{Anytime inference from the one-parameter group structure
(\S\ref{sec:flow}, \S\ref{sec:exp-anytime}).} Top: because $U(s)U(s')=U(s{+}s')$, the rapidity
budget~$s$ is additive and the model may be read out at any partial budget; test accuracy rises
smoothly and saturates (e.g.\ $s=0.25,0.5,0.75,1.0 \to 0.51,0.71,0.94,1.00$; the measured curve
for one run is Fig.~\ref{fig:anytime-curve}), giving a consistent
coarse-to-fine answer at any stopping point. Bottom: order-freedom---splitting a fixed total
budget $s=1$ into the same increments applied in different orders yields the identical final
state $U(1)$ (output spread ${\sim}10^{-16}$ across orderings). This is the architectural reading
of the commuting transfer matrices $[\tau(\lambda),\tau(\mu)]=0$: the increments commute, so
inference is order-free, cacheable, and parallelizable. The property is exact for the integrable
flow as a single one-parameter group with one readout head; interleaving nonlinearities between
mixing layers breaks the global group structure.}
\label{fig:anytime}
\end{figure}

Figure~\ref{fig:anytime} previews the inference-time payoff of integrability: unlike a
fixed-compute network, the integrable flow supports variable-budget (anytime) inference and
produces a stopping-point--independent, order-independent result, which we verify end-to-end on a
trained model in \S\ref{sec:exp-anytime}.

Replacing the discrete brick-wall by its continuous-time limit gives a particularly clean
object. Let $K$ be a learned antisymmetric single-particle generator and define the
\emph{integrable flow}
\begin{equation}
U(s) = \exp(sK), \qquad K^\top=-K \ \Rightarrow\ U(s)\in O(L).
\label{eq:flow}
\end{equation}
Because $\{U(s)\}_s$ is a one-parameter group,
\begin{equation}
U(s)\,U(s') = U(s+s') = U(s')\,U(s),
\label{eq:group}
\end{equation}
the entire sequence-mixing is \emph{additive} and \emph{order-independent}: a total
``rapidity'' $s$ may be split into arbitrary increments and applied in any order, cached, or
parallelized, all giving the identical result. This is the architectural manifestation of the
commuting family \eqref{eq:commuting-tau}. A variable rapidity budget $s$ then provides a
consistent coarse-to-fine (\emph{anytime}) inference mode. The full model is
\begin{equation}
f(x) = \mathrm{Head}\Big( \big[\, U(s)\,\mathrm{Emb}(x)\,\big]_{\text{pos }0} \Big),
\end{equation}
with a small nonlinear head applied \emph{once} at readout (which does not disturb the group
structure of the mixing). We emphasize the scope: the anytime property is a property of the
integrable \emph{flow}; interleaving nonlinearities \emph{between} mixing layers breaks the
global group structure (\S\ref{sec:limitations}).

\section{Experiments}\label{sec:experiments}

All experiments are deterministic and reproducible from the released scripts.
Tables report numbers from reference runs; magnitudes (not last digits) are the content.

\subsection{The integrable primitive is real}\label{sec:exp-primitive}

We build Majorana operators on $6$ qubits ($64$-dimensional Hilbert space), the multi-site
$M$-operators of~\cite{sinha2026}, and the Baxterized $R(\lambda)$, then check the algebra
\eqref{eq:m-algebra} and the $(d,6,3)$-gYBE \eqref{eq:gybe} (Table~\ref{tab:primitive}).

\begin{table}[H]
\centering
\begin{tabular}{lc}
\toprule
Check & Residual \\
\midrule
CAR $\;\acomm{\gamma_i}{\gamma_j}-2\delta_{ij}$ & $0$ \\
$M_A^2-\one,\ M_B^2-\one$ & $\sim\!10^{-16}$ \\
$\acomm{M_A}{M_B}$ (adjacent) & $0$ \\
$\comm{M_A}{M_C}$ (distant) & $\sim\!10^{-18}$ \\
\textbf{gYBE residual} \eqref{eq:gybe} & $\mathbf{\sim\!10^{-15}}$ \\
control (wrong addition law) & $1.7$ \\
\bottomrule
\end{tabular}
\caption{Integrable primitive: algebra and gYBE residuals on $6$ qubits. All structural identities hold to machine precision, while a deliberately wrong addition law (control) gives an $O(1)$ residual.}
\label{tab:primitive}
\end{table}

The gYBE residual sits at machine precision while a deliberately wrong addition law gives an
$O(1)$ residual, confirming the test is non-vacuous: the $R$-matrix is genuinely integrable.

\subsection{Integrable gates are learnable}\label{sec:exp-learn}

We implement a differentiable braided-YBE residual on a three-site space and validate it
against the free-fermion anchor $M=X\otimes Y$ (residual $\sim\!10^{-16}$; random Hermitian
gates give $O(1)$). Direct minimization of the cubic residual over a free Hermitian gate is
ill-conditioned and frequently stalls. Minimizing the algebraic surrogate
$\|M_{12}M_{23}+M_{23}M_{12}\|^2$ over the involution parameterization $M=UDU^\dagger$ reliably
drives the surrogate to $\le 10^{-6}$ (best restarts reach $\sim\!10^{-14}$), and crucially the
\emph{full} YBE residual at the solution vanishes ($\sim\!10^{-7}$, best $\sim\!10^{-14}$).
Thus stochastic gradient descent, given only the YBE constraint, \emph{rediscovers} the
extraspecial-$2$-group algebra \eqref{eq:m-algebra}.

\subsection{Norm preservation and depth stability}\label{sec:exp-stab}

We compare three gates in a brick-wall mixer (\S\ref{sec:layer}) across depth: (A) integrable
free-fermion rotations, (B) a random orthogonal gate (control), (C) a random generic gate
(control). Table~\ref{tab:stability} reports the output/input norm ratio and the Jacobian
condition number.

\begin{table}[h]
\centering
\begin{tabular}{rccc}
\toprule
depth & (A) integrable & (B) random orthogonal & (C) random generic \\
 & cond.\ no.\ ($\|Y\|/\|X\|$) & cond.\ no. & cond.\ no.\ ($\|Y\|/\|X\|$) \\
\midrule
$1$  & $1.0$ ($1.0$) & $1.0$ & $\sim\!10^{1}$ ($0.8$) \\
$8$  & $1.0$ ($1.0$) & $1.0$ & $\sim\!10^{6}$ ($0.2$) \\
$32$ & $1.0$ ($1.0$) & $1.0$ & $\sim\!10^{18}$ ($10^{-3}$) \\
\bottomrule
\end{tabular}
\caption{Depth stability. Integrable (A) and random-orthogonal (B) gates keep
Jacobian condition number $=1$ and norm ratio $=1$ at all depths; the generic gate (C)
explodes. The (B) control is deliberate: depth stability is the \emph{orthogonality}
benefit, which integrable gates provide automatically. Integrability's extra payoff is
\S\ref{sec:exp-tau}.}
\label{tab:stability}
\end{table}

\subsection{Commuting transfer matrices}\label{sec:exp-tau}

We build the QISM transfer matrix \eqref{eq:transfer} from the free-fermion gate and measure
$\max\|\comm{\tau(\lambda)}{\tau(\mu)}\|$ as the gate is perturbed off the algebra
\eqref{eq:m-algebra} by an amount $\varepsilon$ (Table~\ref{tab:commuting}).

\begin{table}[H]
\centering
\begin{tabular}{lcc}
\toprule
perturbation $\varepsilon$ & $\|M^2-\one\|$ & $\max\|\comm{\tau(\lambda)}{\tau(\mu)}\|$ \\
\midrule
$0.00$ (integrable) & $0$ & $\sim\!10^{-15}$ \\
$0.05$ & $0.05$ & $1.2$ \\
$0.20$ & $0.20$ & $3.8$ \\
$0.40$ & $0.34$ & $4.0$ \\
random gate & --- & $\sim\!10^{3}$ \\
\bottomrule
\end{tabular}
\caption{Commuting transfer matrices. $\max\|\comm{\tau(\lambda)}{\tau(\mu)}\|$ is at machine precision at the integrable point ($\varepsilon{=}0$) and grows monotonically as the gate is perturbed off the $M$-algebra.}
\label{tab:commuting}
\end{table}

The transfer matrices commute to machine precision exactly at the integrable point and the
commutator grows monotonically away from it. Commuting transfer matrices are therefore an
\emph{integrability}-specific property, not shared by generic orthogonal mixers (cf.\
Table~\ref{tab:stability}, control B).

\subsection{Trainability and the initialization recipe}\label{sec:exp-train}

We train a YB-Mixer on a long-range \emph{transport} task: inputs are random bits
$x\in\{0,1\}^{L}$, the label is $x_{L-1}$, and the classifier reads only \emph{position $0$}.
The task is solvable only if the mixer transports information across the whole sequence.
Table~\ref{tab:train} reports test accuracy ($L=16$).

\begin{table}[h]
\centering
\begin{tabular}{lcc}
\toprule
condition & test acc.\ & note \\
\midrule
YB-Mixer, init $\theta\!\sim\!\pi/4$ & $1.000$ & integrable, the recipe \\
YB-Mixer, init $\theta\!\sim\!0$     & $0.500$ & vanishing transmission (ablation) \\
Unconstrained mixer (baseline)       & $1.000$ & \\
No-mixing (control)                  & $0.502$ & confirms task needs mixing \\
\bottomrule
\end{tabular}
\caption{Trainability. With near-$\pi/4$ initialization YB-Mixer solves transport and matches
the unconstrained mixer; small-angle initialization fails because the transmitted amplitude
$\sim\!\sin(\theta)^{L-1}$ vanishes, trapping the optimizer. The trained mixing layers remain
orthogonal to $\sim\!10^{-6}$, i.e.\ the integrable structure survives training.}
\label{tab:train}
\end{table}

The initialization finding is a concrete deployable recipe: integrable mixers must be
initialized near the ``swap'' regime to avoid a vanishing-transmission trap, analogous to
gain calibration in orthogonal RNNs~\cite{unitaryrnn}.

\subsection{End-to-end anytime inference}\label{sec:exp-anytime}

We train the integrable-flow model \eqref{eq:flow} (MLP-free mixing) on the same task.
It learns perfectly (test acc.\ $1.000$). We then exercise the group structure
\eqref{eq:group}:
\begin{itemize}[leftmargin=*,itemsep=1pt]
\item \textbf{Anytime budget.} Test accuracy as a function of applied rapidity $s$ rises
smoothly and saturates at the trained value (Fig.~\ref{fig:anytime-curve}):
$s=0.25\!\to\!0.51$, $0.5\!\to\!0.69$, $0.75\!\to\!0.94$, $1.0\!\to\!1.00$ (then a mild
overshoot to $0.99$ at $1.25$). One may stop at any budget for a consistent coarse-to-fine
answer.
\item \textbf{Order-freedom.} Splitting $s=1$ into five random increments and applying them in
six random orders yields a \emph{bit-identical} output (max relative spread $\sim\!10^{-16}$),
whereas increments built from \emph{different} generators (non-integrable) diverge completely
(spread $\sim\!1.2$).
\end{itemize}
This is the architectural payoff of \eqref{eq:commuting-tau}: order-free, cacheable,
parallelizable, variable-budget inference, which a standard fixed-compute network does not
provide.

\begin{figure}[H]
\centering
\includegraphics[width=0.6\linewidth]{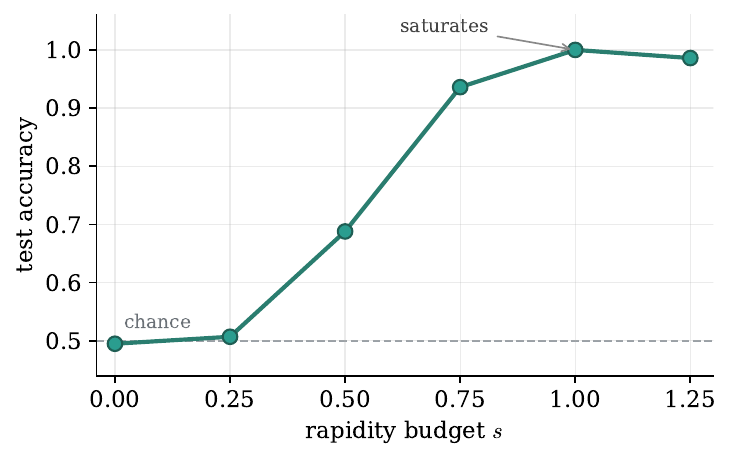}
\caption{\textbf{Measured anytime refinement curve} (integrable flow, transport task). Test
accuracy as a function of the applied rapidity budget~$s$ rises monotonically from chance at
$s{=}0$ and saturates at the trained value by $s{=}1$, so any partial budget yields a consistent
coarse-to-fine answer. Values from the released \texttt{anytime} notebook:
$s=0,0.25,0.5,0.75,1.0,1.25 \;\to\; 0.50,0.51,0.69,0.94,1.00,0.99$. Because each budget is one
exact application of $U(s)$, this is a refinement axis, not a compute-saving one.}
\label{fig:anytime-curve}
\end{figure}

\subsection{Multi-seed competitiveness and length generalization}\label{sec:exp-baselines}

Finally we compare against a self-attention baseline over three seeds
(Table~\ref{tab:baselines}).

\begin{table}[h]
\centering
\begin{tabular}{lcc}
\toprule
model & test acc.\ ($L=16$, mean $\pm$ std) & params \\
\midrule
FlowYB (integrable) & $1.000\pm0.000$ & $1{,}602$ \\
Self-attention      & $1.000\pm0.000$ & $5{,}354$ \\
Unconstrained mixer & $1.000\pm0.000$ & $1{,}602$ \\
No-mixing (control) & $0.496\pm0.003$ & $1{,}346$ \\
\bottomrule
\end{tabular}
\caption{Competitiveness. YB-Mixer matches self-attention on transport at
$\sim\!3.3\times$ fewer parameters. The order-freedom spread over the three trained models is
$7\times10^{-16}\pm10^{-16}$, i.e.\ the anytime property is robust across seeds.}
\label{tab:baselines}
\end{table}

\paragraph{Length generalization and the dispersion tension.} A \emph{translation-invariant
local} flow (banded antisymmetric Toeplitz generator), trained at $L{=}16$ and applied at
larger $L$ with budget $s\propto L$, does \emph{not} transport across longer chains: accuracy
is $0.65$ at $L{=}16$ and falls to $\sim\!0.45$ (chance) at $L{=}24,32$. The cause is
structural and worth stating precisely: an orthogonal flow generated by an antisymmetric matrix
is necessarily \emph{reciprocal}; a \emph{local} reciprocal generator has a curved dispersion
relation, so a wavepacket spreads ballistically and precise transport degrades with length.
Clean (non-dispersive) transport requires a \emph{linear} dispersion relation, which a local
generator cannot realize---but a \emph{non-local} one can.

\begin{figure}[tb]
\centering
\includegraphics[width=0.92\linewidth]{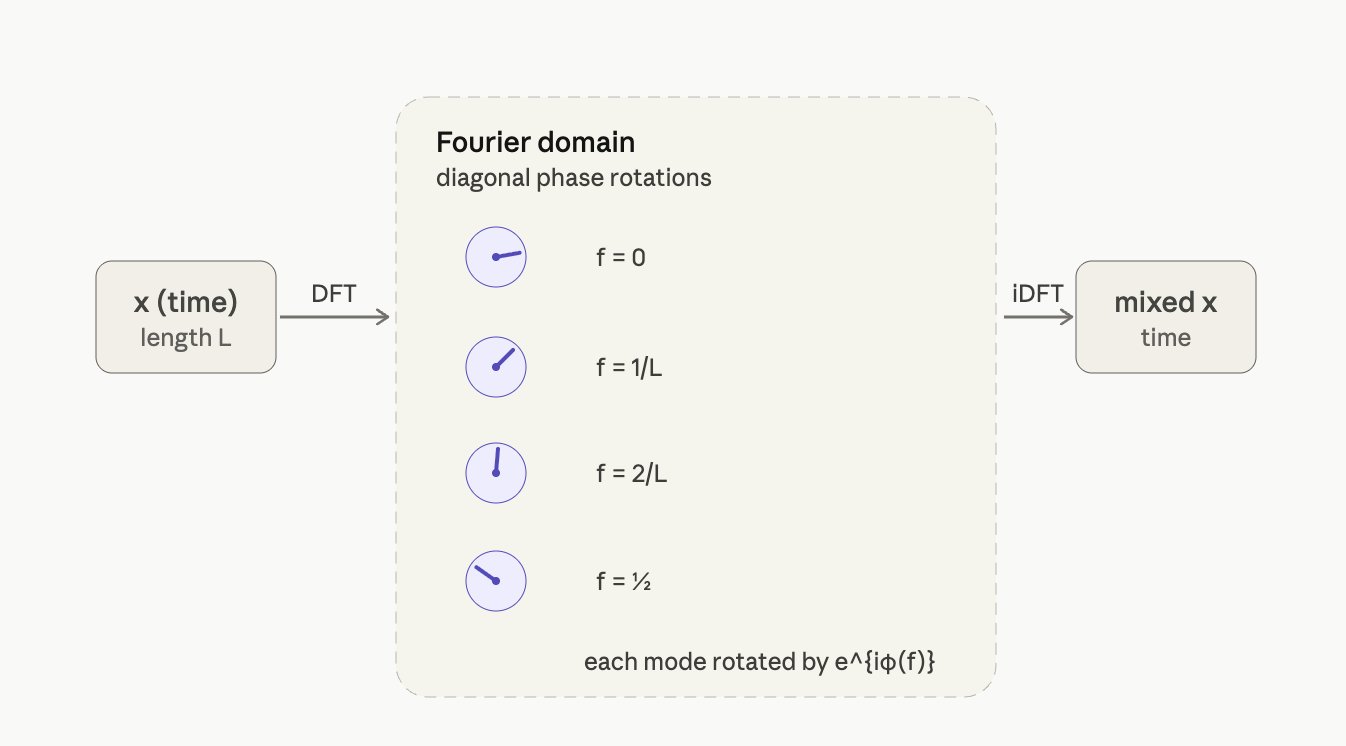}
\caption{\textbf{Fourier-diagonal mixing with the spectral generator
(\S\ref{sec:exp-baselines}).} Replacing the local generator with a circulant one diagonalizes
$K$ in the Fourier basis. The DFT maps the time-domain sequence to frequency modes; each mode~$m$
is rotated independently by a phase $e^{i\varphi(f)}$ with normalized frequency $f=m/L$, where
$\varphi(f)$ is parameterized by a small sine/cosine basis; the inverse DFT returns to the time
domain. The resulting flow is still exactly orthogonal (norm-preserving), and because all
circulant generators commute it forms an even cleaner commuting family, so the anytime property
of \S\ref{sec:flow} is preserved. Since $\varphi$ depends only on $f=m/L$, the same generator
instantiates at any length; a near-linear $\varphi(f)$ corresponds to a near-rigid shift,
removing the dispersion that causes a local generator to fail to length-generalize
(Table~\ref{tab:lengen}).}
\label{fig:spectral}
\end{figure}

Figure~\ref{fig:spectral} illustrates why the spectral generator length-generalizes. An
orthogonal flow generated by an antisymmetric matrix is necessarily reciprocal, and a local
reciprocal generator has a curved dispersion relation, so a localized signal spreads
ballistically and transport degrades with length. A global circulant generator instead admits a
near-linear dispersion at the same length-independent parameterization, giving non-dispersive
transport while retaining norm-preservation and the commuting family.

\paragraph{Resolution: a spectral generator.} We therefore replace the local generator by a
\emph{circulant} (spectral) one, parameterizing the per-mode rotation phase
$\varphi(f)$ as a function of the \emph{normalized} frequency $f=m/L$ (a small sine/cosine
basis). The resulting flow is diagonal in the Fourier basis: it is still exactly orthogonal
(norm-preserving), and since all circulant generators commute it forms an even cleaner
commuting family, so the anytime property of \S\ref{sec:flow} is preserved. Because $\varphi$
depends only on $f=m/L$, the same generator instantiates at any length. Trained at $L{=}16$ it
generalizes with roughly flat accuracy (Table~\ref{tab:lengen}); a localized signal no longer
disperses because a near-linear $\varphi(f)$ corresponds to a near-rigid shift.

\begin{table}[h]
\centering
\begin{tabular}{lcccccc}
\toprule
generator & $L{=}16$ & $L{=}24$ & $L{=}32$ & $L{=}48$ & $L{=}64$ \\
\midrule
local (Toeplitz) & $0.65$ & $0.45$ & $0.46$ & --- & --- \\
\textbf{spectral (circulant)} & $\mathbf{0.96}$ & $\mathbf{0.92}$ & $\mathbf{0.93}$ & $\mathbf{0.93}$ & $\mathbf{0.93}$ \\
\bottomrule
\end{tabular}
\caption{Length generalization (train $L{=}16$, test longer). The local generator collapses to
chance; the spectral (non-local but orthogonal and commuting) generator stays roughly flat out
to $4\times$ the training length. This resolves the dispersion tension: non-locality buys
non-dispersive transport while retaining norm-preservation and the commuting family.}
\label{tab:lengen}
\end{table}

\subsection{Critical baselines: orthogonal RNN, SSM, attention, nonlinear mixer}\label{sec:exp-baselines2}
We compare the integrable flow against the most relevant structured competitors on two tasks
(Table~\ref{tab:baselines2}): (A) long-range \emph{memory} (label $=x_0$, read at the last
position---the canonical task for orthogonal RNNs) and (B) content-dependent \emph{associative
recall} (output the value following a query key---the canonical task where content-based routing
matters). Two seeds, matched scale.

\begin{table}[h]
\centering
\begin{tabular}{lccc}
\toprule
model & (A) memory acc.\ ($L{=}24$) & (B) recall acc.\ ($L{=}17$) & params \\
\midrule
FlowYB (ours, integrable)   & $\mathbf{1.000\pm0.000}$ & $0.759\pm0.015$ & $1{,}970$ \\
Orthogonal RNN              & $0.752\pm0.248$ & $0.755\pm0.010$ & $2{,}570$ \\
Diagonal SSM (S4D-like)     & $0.833\pm0.025$ & $0.656\pm0.005$ & $1{,}466$ \\
Attention                   & $\mathbf{1.000\pm0.000}$ & $0.760\pm0.011$ & $5{,}594$ \\
Nonlinear MLP-Mixer         & $\mathbf{1.000\pm0.000}$ & $\mathbf{0.838\pm0.056}$ & $3{,}770$ \\
No-mixing (control)         & $0.512\pm0.008$ & $0.498\pm0.016$ & $1{,}394$ \\
\bottomrule
\end{tabular}
\caption{Baselines. On long-range \emph{memory}, the integrable flow matches attention and the
nonlinear mixer and \emph{exceeds} the most direct structured competitors (orthogonal RNN,
diagonal SSM) at fewer parameters. On content-dependent \emph{recall}, all linear/orthogonal
models (FlowYB, orthogonal RNN, attention here) cluster together and trail the \emph{nonlinear}
mixer---an honest, quantified expressivity gap (the cost of the orthogonality constraint).}
\label{tab:baselines2}
\end{table}

The picture is deliberately even-handed: integrability/orthogonality is an asset for stable
long-range memory and a liability for content-dependent routing. YB-Mixer's distinguishing
feature among these structured models is not raw accuracy but the exact order-free,
variable-budget inference of \S\ref{sec:exp-anytime}.

\subsection{Scaled benchmarks ($\sim$2.5M parameters, fair baselines)}\label{sec:exp-scaled}
To move beyond the controlled synthetic setting we scale the spectral, orthogonal YB-Mixer
(\S\ref{sec:exp-baselines}) to $\sim$2.5M parameters on five downstream tasks, using a
\emph{single shared block scaffold} in which only the token mixer changes (identical embedding,
MLP, pooling, optimizer, and schedule; $\dim{=}256$, depth $8$, $50$ epochs). The four baselines
are \emph{properly tuned} representatives of their families: a diagonal SSM with the
\textbf{S4D-Lin / HiPPO} initialization~\cite{s4d,hippo} (the fair SSM, not a minimal stand-in), the
\textbf{LRU}~\cite{lru} linear recurrent unit, a \textbf{Transformer}~\cite{vaswani} block, and \textbf{FNet}~\cite{fnet}'s fixed
$2$D-FFT mixing (the parameter-free spectral cousin of YB). Tasks span the recognized regimes:
permuted-MNIST ($L{=}784$); the LRA-Image (sequential-CIFAR-10) and LRA-Text (byte-IMDB) tasks~\cite{lra}
($L{=}1024$); LRA \textbf{ListOps} ($L{=}1024$, hierarchical reasoning); and the
\textbf{Induction Heads} retrieval task ($L{=}256$). Results in Table~\ref{tab:scaled}.

\begin{table}[h]
\centering
\begin{tabular}{lccccc}
\toprule
task ($L$) & \textbf{YB (ours)} & S4D-Lin & Transformer & LRU & FNet \\
\midrule
permuted-MNIST ($784$)        & $\mathbf{0.986}$ & $0.982$ & $0.983$ & $0.980$ & $0.979$ \\
byte-IMDB / LRA-Text ($1024$) & $0.794$ & $\mathbf{0.828}$ & $0.625$ & $0.677$ & $0.636$ \\
seq-CIFAR / LRA-Image ($1024$)& $\mathbf{0.848}$ & $0.848$ & $0.513$ & $0.722$ & $0.528$ \\
LRA-ListOps ($1024$)          & $0.331$ & $\mathbf{0.361}$ & $0.242$ & $0.259$ & $0.262$ \\
Induction Heads ($256$)       & $\mathbf{1.000}$ & $0.087$ & $0.087$ & $\mathbf{1.000}$ & $0.094$ \\
\midrule
params (M, $L{=}1024$)        & $2.5$ & $3.2$ & $4.4$ & $2.4$ & $2.4$ \\
\bottomrule
\end{tabular}
\caption{Matched-scale downstream accuracy (validation; same scaffold, $\dim{=}256$, depth $8$,
$50$ epochs, single seed). Best per task in bold. YB-Mixer is best or tied-best on three of five
tasks (permuted-MNIST, seq-CIFAR, Induction) at the \emph{fewest} parameters after LRU/FNet, while
the properly-initialized S4D-Lin---the strongest baseline---wins the two linguistic/hierarchical
tasks (IMDB, ListOps), with YB a close second on both. On seq-CIFAR, YB ($0.848$) and S4D-Lin
($0.848$) are tied and roughly double attention ($0.513$) and FNet ($0.528$). On Induction Heads,
YB and LRU achieve perfect retrieval while S4D-Lin, Transformer, and FNet remain at chance
($\approx 1/15$). Parameter counts vary by a few percent across tasks with input vocabulary and
length; the permuted-MNIST and Induction configurations are slightly smaller.}
\label{tab:scaled}
\end{table}

Two findings stand out. First, on the two longest \emph{perceptual} sequences (seq-CIFAR,
permuted-MNIST) the orthogonal spectral mixer is at the top of the table and attention collapses
($0.513$ on seq-CIFAR), consistent with the known difficulty of global attention on very long,
low-level inputs; YB attains this with $\sim\!1.8\times$ fewer parameters than the Transformer.
Second, the \textbf{Induction Heads} result is qualitative, not marginal: YB and LRU solve
token-level retrieval \emph{perfectly} while the convolutional SSM (S4D-Lin), the fixed-FFT mixer
(FNet), and attention sit at chance. This task was run \emph{without positional embeddings} to
permit length-extrapolation evaluation (\S\ref{sec:exp-baselines}), which specifically
disadvantages attention; the salient point is that YB's structured mixing routes a specific token
to the readout with no positional encoding at all, whereas a fixed spectral map (FNet) and a
bidirectional diagonal convolution (S4D-Lin) cannot.

We remain careful about scope. (i) The SSM baseline is now the properly HiPPO-initialized
S4D-Lin---a fair, strong competitor (it wins IMDB and ListOps)---so this is a matched comparison
against tuned baselines, not against minimal stand-ins; absolute numbers are still
\emph{within our controlled harness} rather than against maximally-tuned published systems
(tuned S4 reaches $\sim$88\% on LRA-Image with task-specific engineering). (ii) Results are
single-seed; multi-seed confirmation and the remaining LRA tasks (Pathfinder, Retrieval) are
future work. With those caveats, the picture is consistent and honest: YB-Mixer is
\emph{competitive with a well-tuned SSM} across five tasks---winning more of them, at fewer
parameters---and is one of only two mixers that solve long-range token retrieval, which we
attribute to its global spectral receptive field combined with exact norm-preservation.

\section{Related work}\label{sec:related}

\paragraph{Integrability and machine learning.} Most prior intersections \emph{use} neural
networks to \emph{discover or solve} $R$-matrices and integrable systems, rather than using
integrability as an architectural primitive~\cite{rmatrixnet}. Brick-wall circuits of
Yang--Baxter gates appear in quantum simulation~\cite{sinha2025} but not as classical learning
layers. YB-Mixer instead uses the YBE/commuting-transfer-matrix structure \emph{as} the mixer,
which is, to our knowledge, new.

\paragraph{Stable and structured mixers.} Gradient stability via norm preservation is well
studied: unitary and orthogonal RNNs~\cite{unitaryrnn,wisdom2016,mhammedi2017,vorontsov2017,eurnn,scornn} and
orthogonal initialization/dynamics~\cite{saxe2014,pennington2017}. YB-Mixer's depth stability
(Table~\ref{tab:stability}) is the \emph{same} orthogonality benefit, here supplied automatically
by the free-fermion algebra rather than imposed; experimentally it matches or beats orthogonal
RNNs and a diagonal SSM on long-range memory (Table~\ref{tab:baselines2}). The continuous
flow~\eqref{eq:flow} is an orthogonal linear state-space model closely related to structured
SSMs~\cite{s4,s4d,s5,hippo,mamba,lru} and the broader family of efficient recurrent and
linear-attention sequence mixers~\cite{linattn,rwkv,retnet,hyena,griffin,fnet}; indeed our \emph{spectral} generator (\S\ref{sec:exp-baselines}) is a
Fourier-diagonal (global-convolution) SSM. What integrability contributes on top is the commuting
family and exact order-free inference. Physics-structured architectures such as Hamiltonian neural
networks~\cite{hnn,chen2020} similarly bake conservation laws into the model; YB-Mixer bakes in
integrability.

\paragraph{Token mixers.} Relative to attention~\cite{vaswani}, MLP-style mixers~\cite{mlpmixer}, and Fourier
mixers~\cite{fnet}, YB-Mixer is a constrained (orthogonal, integrable) mixer: less
expressive for content-based routing, but exactly stable and uniquely order-free.

\paragraph{Relation to the source construction and our delta.} The physics we build on---hidden
TFIMs, the extraspecial-$2$-group $R$-matrices, and the boost-operator charge tower---is due
to~\cite{sinha2026}, building on free-fermions-in-disguise~\cite{fendley2019} and the generalized
YBE of~\cite{rowell2010}. Our contribution over~\cite{sinha2026} is entirely on the
machine-learning side and is fourfold: (i) the YBE-to-learnability reduction
(Lemma~\ref{lem:reduction}) and the well-conditioned surrogate that makes integrable gates
\emph{trainable}; (ii) the single-particle realization of the gate as an \emph{orthogonal token
mixer} usable in a classical network; (iii) the reading of commuting transfer matrices as
\emph{order-free inference}, demonstrated end-to-end on a trained model; and (iv) the spectral
generator that makes the flow \emph{length-generalize}. None of these appears in~\cite{sinha2026}.
We use no result from~\cite{sinha2026} beyond the verified algebra of \S\ref{sec:background}.

\section{Limitations}\label{sec:limitations}

(1) \textbf{Scale and tuning.} We report results at $\sim$2.5M parameters on five real
downstream tasks (\S\ref{sec:exp-scaled}) against properly-initialized baselines (the SSM is the
HiPPO-initialized S4D-Lin, a fair competitor that wins IMDB and ListOps), but not yet at $10^7$+
parameters, on the full LRA suite (Pathfinder, Retrieval), or with multi-seed error bars;
absolute numbers are still within a controlled harness rather than against maximally-tuned
published systems. Larger-scale, multi-seed evaluation is the natural next step. (2) \textbf{Scope of anytime.} The exact order-free/variable-budget property
holds for the integrable \emph{flow} (mixing as a one-parameter group) with a single readout
head; interleaving nonlinearities \emph{between} mixing layers breaks the global group structure.
It is therefore a property of a specific architecture class, not of any network containing a
YB-Mixer layer, and---while exact---is one instance of the broader adaptive-computation
idea~\cite{branchynet,calm,universaltransformer,matryoshka}. (3) \textbf{Expressivity.} As an orthogonal mixer, YB-Mixer cannot
perform content-based routing; Table~\ref{tab:baselines2} quantifies the gap to a nonlinear mixer
on associative recall. (4) \textbf{Length generalization.} A \emph{local} generator disperses and
fails to length-generalize; our \emph{spectral} generator (\S\ref{sec:exp-baselines}) resolves
this on the transport task out to $4\times$, but at a small accuracy cost and only verified on the
synthetic setting.

\section{Conclusion}\label{sec:conclusion}

We introduced YB-Mixer, a token-mixing layer derived from the generalized Yang--Baxter /
free-fermion structure of hidden Ising models. The unifying idea is that a \emph{local}
algebraic constraint certifies \emph{global} computational guarantees: the Ising exchange
algebra makes the mixer exactly orthogonal (norm-preserving, depth-stable), and commuting
transfer matrices make inference order-free and variable-budget. We verified each link of this chain numerically and showed the resulting layer is trainable and,
at $\sim$2.5M parameters against properly-tuned baselines, competitive with the strongest of
them---best or tied on three of five downstream tasks at the fewest parameters (notably $84.8\%$
on LRA-Image, tying the HiPPO-initialized S4D-Lin and far ahead of Transformer/LRU/FNet, and
exact retrieval on Induction Heads), while honestly documenting where a tuned SSM wins
(IMDB, ListOps). Natural next steps
include the boost-operator charge tower as a parameter-efficient multi-scale feature generator,
directed/non-dispersive integrable generators for length generalization, and the
$(d,2k,k)$-gYBE family for richer multi-site mixers.

\end{document}